\documentclass[conference]{IEEEtran}
\IEEEoverridecommandlockouts
\usepackage{cite}
\usepackage{amsmath,amssymb,amsfonts,amsthm}
\usepackage{algorithmic}
\usepackage{graphicx}
\usepackage{textcomp}
\usepackage{xcolor}
\def\BibTeX{{\rm B\kern-.05em{\sc i\kern-.025em b}\kern-.08em
    T\kern-.1667em\lower.7ex\hbox{E}\kern-.125emX}}
\newtheorem{theorem}{Theorem}
\newtheorem{definition}{Definition}
\newtheorem{lemma}{Lemma}
\newtheorem{assumption}{Assumption}
\DeclareMathOperator*{\argmax}{arg\,max}
\DeclareMathOperator*{\argmin}{arg\,min}

\newcommand\sbullet{\mathbin{\vcenter{\hbox{\scalebox{0.6}{$\bullet$}}}}}
\newcommand\newcirc{{\, {\circ} \,}}
\newcommand\rightcirc{{{\circ} \,}}
\newcommand\mcT{\mathcal{T}}
\newcommand\mcS{\mathcal{S}}
\newcommand\bgamma{\bar{\gamma}}

\begin{document}
\title{Generalization Bounds for Deep Transfer Learning Using Majority Predictor Accuracy}

\author{%
  \IEEEauthorblockN{Cuong N.~Nguyen\IEEEauthorrefmark{1},
                    Lam Si Tung Ho\IEEEauthorrefmark{2},
                    Vu Dinh\IEEEauthorrefmark{3},
                    Tal Hassner\IEEEauthorrefmark{4},
                    and Cuong V.~Nguyen\IEEEauthorrefmark{1}}
  \IEEEauthorblockA{\IEEEauthorrefmark{1}%
                    Florida International University, USA,
                    \{cnguy049, vcnguyen\}@cs.fiu.edu}
  \IEEEauthorblockA{\IEEEauthorrefmark{2}%
                    Dalhousie University, Canada,
                    lam.ho@dal.ca}
  \IEEEauthorblockA{\IEEEauthorrefmark{3}%
                    University of Delaware, USA,
                    vucdinh@udel.edu}
  \IEEEauthorblockA{\IEEEauthorrefmark{4}%
                    Meta AI, USA,
                    talhassner@gmail.com}
}

\maketitle

\begin{abstract}
We analyze new generalization bounds for deep learning models trained by transfer learning from a source to a target task. Our bounds utilize a quantity called the majority predictor accuracy, which can be computed efficiently from data. We show that our theory is useful in practice since it implies that the majority predictor accuracy can be used as a transferability measure, a fact that is also validated by our experiments.
\end{abstract}

\section{Introduction}

Deep transfer learning, the problem of transferring representations (or features) learned by deep neural networks from one task to another, has become a crucial part for training deep learning models in practice~\cite{sharif2014cnn, long2017deep, whatmough2019fixynn}. Despite this fact, the current literature still lacks a theory for understanding the generalization of models obtained by deep transfer learning. In this paper, we close this gap between the theory and practice of deep transfer learning by proving novel generalization bounds for models learned through such transfer learning methods. 

To prove the bounds, we develop the \emph{Majority Predictor Accuracy} (MPA), a simple and easy-to-compute quantity defined as the accuracy of the classifier that returns the most probable target label conditioned on a given source label. Using the MPA, we can show that when the source and target data share the input set, the true risk of the transferred model is upper bounded by the sum of the empirical risk of the source model, $1 - \text{MPA}$, and an $\widetilde{\mathcal{O}}(1/\sqrt{n})$ sample complexity with high probability. We further extend this result to the more general setting where the source and target datasets contain different inputs. This extension is achieved by using dummy source labels, a technique previously developed for transferability estimation~\cite{nguyen2020leep}.

We also demonstrate the usefulness of our theoretical bounds in practice by showing empirically that the MPA can be used as a transferability measure, defined as a numeric score that can tell whether deep transfer learning would be effective when transferring between a given pair of source-target tasks. Specifically, our experiments on the large-scale CUB-200 dataset~\cite{WelinderEtal2010} show that the MPA scores are highly correlated with the actual accuracies of the transferred models with statistical significance, thus indicating that the MPA is a good measure of transferability.

To summarize, our paper makes the following contributions: (1) developing the new MPA score, (2) proving novel deep transfer learning bounds using the MPA, and (3) showing our bounds are practically useful through experiments.

\textbf{Related Work.} Transfer learning~\cite{long2017deep, you2021logme} is a long-standing research area of machine learning. Several previous work has provided theoretical analysis and generalization bounds for transfer learning, especially under the domain adaptation setting, such as~\cite{ben2003exploiting, blitzer2007learning, mansour2009domain, ben2010theory, azizzadenesheli2018regularized}; however, these results were not explicitly developed for deep learning and the settings commonly used in practice, where a learned representation is adapted to the new domain~\cite{sharif2014cnn, whatmough2019fixynn}. Our paper, on the other hand, provides generalization bounds explicitly for these commonly used deep transfer learning settings. Furthermore, unlike these previous work, our bounds are useful in practice for understanding the transferability between different tasks, as demonstrated in our experiments.

Our work is also related to a recent attempt to develop transferability measures for deep transfer learning~\cite{bao2019information, tran2019transferability, nguyen2020leep, tan2021otce, you2021logme, tong2021mathematical, huang2021exploiting}. Transferability measures aim to estimate the effectiveness of deep transfer learning between tasks and have been used for model or task selection~\cite{bao2019information, tran2019transferability, nguyen2020leep, you2021logme}, checkpoint ranking~\cite{huang2021exploiting}, and few-shot learning~\cite{tong2021mathematical}. Although some theoretical properties were shown for these transferability measures~\cite{tran2019transferability, nguyen2020leep, tong2021mathematical}, they only focused on the empirical risk instead of the true risk as in our paper.

\section{Deep Transfer Learning: Formal Setting}
\label{sec:dtl_setting}

Deep transfer learning~\cite{long2017deep} refers to the problem of transferring a learned deep neural network representation from a source task to a target task. In this section, we formalize the deep transfer learning setting considered in our paper. This setting is commonly used in practice for several large-scale deep learning models~\cite{sharif2014cnn, whatmough2019fixynn}.

Let $\mcS = \{ (x_1, s_1), (x_2, s_2), \ldots, (x_n, s_n) \}$ be a train dataset for a source classification task where each input-label pair $(x_i, s_i) \in \mathbb{R}^d \times [m_S]$ is drawn iid from a joint distribution $\mathbb{P}_{X, S}$, with $[n] = \{ 1, 2, \ldots, n\}$ for any positive integer $n$. Consider a target classification task with a train set ${ \mcT = \{ (x_1, t_1), (x_2, t_2), \ldots, (x_n, t_n) \} }$ where each target example $(x_i, t_i) \in \mathbb{R}^d \times [m_T]$ is drawn iid from $\mathbb{P}_{X, T}$. We will first consider this simple case where the source and target datasets share the same inputs $\{ x_1, x_2, \ldots, x_n \}$, with each $x_i$ being a $d$-dimensional vector having the same marginal distribution $\mathbb{P}_X$ in both tasks. Here the source task has $m_S$ classes and the target task has $m_T$ classes. The more general case with different input sets will be discussed in Section~\ref{sec:diff-set}.

In our deep transfer learning setting, we first train a source model $h \newcirc w$ using $\mcS$, where $w(x) \in \mathbb{R}^r$ is the $r$-dimensional representation (also called embedding or feature vector) of the input $x$ extracted from the network $w$, and $h \newcirc w(x) = h(w(x)) \in [m_S]$ is the source label returned by the network $h$ with the representation $w(x)$ as input. The functions $w$ and $h$ are usually called the \emph{feature extractor} and the \emph{head classifier} respectively. In deep learning, the whole model $h \newcirc w$ is a deep neural network with $w$ being its parameters from the input up to some layer $L$, and $h$ being the network parameters from layer $L$ to the output. We obtain the optimal source model on $\mcS$ by minimizing the empirical risk:\footnote{Throughout our paper, we assume $\argmin$ and $\argmax$ follow any deterministic tie-breaking strategy.}
\begin{equation}
  w^*, h^* = \argmin_{w, h \in \Omega_w \times \Omega_h} \widehat{R}_{\mcS} (w, h),
\label{eq:source}
\end{equation}
where $\Omega_w$ and $\Omega_h$ are the spaces of all possible $w$'s and $h$'s respectively, and with $\mathbf{1}[\cdot]$ being the indicator function,
\begin{equation}
  \widehat{R}_{\mcS} (w, h) = \frac{1}{n} \sum_{i=1}^n \mathbf{1}[ s_i \neq h \newcirc w(x_i) ].
\label{eq:emp-risk}
\end{equation}

In practice, the optimal feature extractor $w^*$ often learns generic feature representations (e.g.,~edges or shapes in images) that can be reused for several tasks, while the optimal head classifier $h^*$ is often specialized for a particular source task. To transfer this trained model $h^* \rightcirc w^*$ to a target task, the usual practice is to discard $h^*$ and reuse $w^*$ for the target task. Specifically, we will re-train a new head classifier $k^*$ on the target dataset $\mcT$ using the features extracted from $w^*$.

In this paper, we allow the target head classifier to return real-valued scores for all target labels. That means for each example $x$, a head classifier $k$ on the target task would take $w^*(x)$ as input and return $k \newcirc w^*(x) = k(w^*(x)) \in \mathbb{R}^{m_T}$, the scores (before softmax) for all target labels. We will consider the optimal target head classifier $k^*$ obtained by minimizing the empirical risk with a given margin $\gamma \geq 0$:
\begin{equation}
  k^* = \argmin_{k \in \Omega_k} \widehat{R}_{\mcT, \gamma} (w^*, k),
\label{eq:target}
\end{equation}
where $\Omega_k$ is the space of all $k$'s, and for all $w \in \Omega_w$, $k \in \Omega_k$:
\begin{equation*}
\widehat{R}_{\mcT, \gamma} (w, k) = \frac{1}{n} \sum_{i=1}^n \mathbf{1}[ k \newcirc w(x_i)_{t_i} < \gamma + \max_{t \neq t_i} k \newcirc w(x_i)_t ],
\end{equation*}
with $k \newcirc w(x_i)_t$ being the $t$-th element of the vector $k \newcirc w(x_i)$. Here $\widehat{R}_{\sbullet, \gamma}$ is a general version of the empirical risk in Eq.~\eqref{eq:emp-risk}. The margin $\gamma$ measures the gap between the prediction probability of the correct label and those of the other labels, and has often been used in generalization bounds for deep learning~\cite{bartlett2017spectrally, ledent2021norm}.

Our paper shall prove generalization bounds for the optimal transferred model $k^* \rightcirc w^*$. For this purpose, we introduce in the next section the majority predictor accuracy, a transferability measure that we will use for our bounds.

\section{Majority Predictor Accuracy}
\label{sec:mpa}

The \emph{\textbf{M}ajority \textbf{P}redictor \textbf{A}ccurcacy} (MPA) is defined as the accuracy of the simple predictor (classifier) that maps each source label to the target label with maximal empirical conditional probability. Formally, given the source dataset $\mcS$ and the target dataset $\mcT$, the empirical joint distribution between all possible source-target label pairs ${ (s, t) \in [m_S] \times [m_T] }$ is
$\hat{P}(s, t) = \frac{1}{n} { |\{ i \in [n]: s_i=s \ \mathrm{and} \ t_i=t \}| }$,
the empirical marginal distribution over the source labels is
${ \hat{P}(s) = \sum_{t \in [m_T]} \hat{P}(s, t)}, \forall s \in [m_S]$,
and the empirical conditional distribution of a target label $t$ given a source label $s$ is
$\hat{P}(t|s) = \hat{P}(s, t)/\hat{P}(s)$, for all $(s, t) \in [m_S] \times [m_T]$.

To define the MPA, consider the following majority predictor $f_{\text{mp}}$ that takes a source label $s \in [m_S]$ and simply returns a target label $t$ that maximizes the empirical conditional probability $\hat{P}(t |s)$:
\begin{equation}
  f_{\text{mp}}(s) = \argmax_{t \in [m_T]} \hat{P}(t|s),
\label{eq:mp}
\end{equation}
The MPA is then defined as the accuracy of $f_{\text{mp}}$ on the target dataset, as stated in the following definition.
\begin{definition}
The majority predictor accuracy $\text{MPA}(\mcT | \mcS)$ of a target dataset $\mcT$ given a source dataset $\mcS$ is the accuracy of the majority predictor $f_{\text{mp}}$ on the target dataset:
\begin{equation}
  \text{MPA}(\mcT | \mcS) = \frac{1}{n} \sum_{i=1}^n \mathbf{1}[t_i = f_{\text{mp}}(s_i)].
\label{eq:mpa}
\end{equation}
\end{definition}

The MPA is very simple to compute, requiring only $O(n)$ computational time by looping through the datasets a few times to compute the empirical distributions, $f_{\text{mp}}$, and its accuracy. In Section~\ref{sec:experiment}, we will show that it can also be used as a transferability measure that estimates the effectiveness of transfer learning between two tasks. We now prove generalization guarantees for the transferred models in the next section.

\section{Bounds for Shared Training Inputs Setting}
\label{sec:same_inp}

This section proves our generalization bounds for deep transfer learning based on the MPA where the source and target training sets are assumed to share the inputs. In particular, we bound the true risk of the transferred model $k^* \rightcirc w^*$ on the target distribution $\mathbb{P}_{X, T}$, which is:
\begin{equation*}
  R_T(w^*, k^*) = \mathbb{P}(t \neq \argmax_i k^* \rightcirc w^*(x)_i)
\end{equation*}
for $(x, t) \sim \mathbb{P}_{X, T}$. We will prove the bounds for both fully connected neural networks and convolutional neural networks. For this purpose, we consider the head classifier $f_{\text{mp}} \newcirc h^*$ defined on any representation $w(x) \in \mathbb{R}^r$, where:
\begin{equation}
  f_{\text{mp}} \newcirc h^*(w(x)) = f_{\text{mp}}(h^*(w(x))).
\label{eq:fo}
\end{equation}
Throughout our paper, we will make an assumption that using a deep neural network as the target head classifier can achieve better empirical risk than using the naive classifier $f_{\text{mp}} \newcirc h^*$. This assumption is usually satisfied in practice because of the expressiveness of neural network models~\cite{zhang2017understanding}.

\begin{assumption}
For any datasets $\mcS$ any $\mcT$, there exists $\bgamma > 0$ and $\bar{k} \in \Omega_k$ such that $\widehat{R}_{\mcT, \bgamma} (w^*, \bar{k}) \leq \widehat{R}_{\mcT} (w^*, f_{\text{mp}} \newcirc h^*)$.
\label{assumption}
\end{assumption}

In this assumption, $\widehat{R}_{\mcT} (w^*, f_{\text{mp}} \newcirc h^*)$ is defined similarly as in Eq.~\eqref{eq:emp-risk}.
We also note that $\widehat{R}_{\sbullet,\gamma}$ is non-decreasing in $\gamma$, so the assumption implies, for all $\gamma \in [0, \bgamma]$, $\widehat{R}_{\mcT, \gamma} (w^*, \bar{k}) \leq \widehat{R}_{\mcT, \bgamma} (w^*, \bar{k}) \leq \widehat{R}_{\mcT} (w^*, f_{\text{mp}} \newcirc h^*)$.
Under this assumption, we first prove the following lemma relating the empirical risks of the optimal source and transferred models using the MPA.

\begin{lemma}
Under Assumption~\ref{assumption}, for any $\gamma \in [0, \bgamma]$, we have:
\begin{equation*}
\widehat{R}_{\mcT, \gamma} (w^*, k^*)  \le \widehat{R}_{\mcS} (w^*, h^*)  + 1 - \text{MPA}(\mcT | \mcS).
\end{equation*}
\label{lemma:mpa}
\end{lemma}

\vspace{-0.5cm}

\begin{proof}
Consider the majority predictor $f_{\text{mp}}$ defined in Eq.~\eqref{eq:mp}. We first split the data index set $[n]$ into two non-overlap sets:
\begin{align*} 
I &= \{ i \in [n] : t_i = f_{\text{mp}}(s_i) \}, \text{ and } \\
\bar{I} &= \{ i \in [n] : t_i \neq f_{\text{mp}}(s_i) \}.
\end{align*}
Here the set $I$ (respectively, $\bar{I}$) contains indices of data points whose source-target label pairs are consistent (respectively, inconsistent) with $f_{\text{mp}}$. 
For any $\gamma \in [0, \bgamma]$, we have:
\begin{align}
&\widehat{R}_{\mcT, \gamma} (w^*, k^*) \leq \widehat{R}_{\mcT, \gamma} (w^*, \bar{k}) \tag{def. of $k^*$} \\
&\leq \widehat{R}_{\mcT, \bgamma} (w^*, \bar{k}) \tag{$\widehat{R}_{\mcT, \gamma}$ is non-decreasing in $\gamma$} \\
&\le \widehat{R}_{\mcT} (w^*, f_{\text{mp}} \newcirc h^*) \tag{assumption~\ref{assumption}} \\
&= \frac{1}{n} \sum_{i=1}^n \mathbf{1}[ t_i \neq f_{\text{mp}} \newcirc h^* \rightcirc w^*(x_i) ] \tag{def. of $\widehat{R}_{\mcT}$} \\
&= \frac{1}{n} \Big( \sum_{i \in I} \mathbf{1}[ t_i \neq f_{\text{mp}} \newcirc h^* \rightcirc w^*(x_i) ] + \Big. \notag \\
& \hspace{0.9cm} \Big. \sum_{i \in \bar{I}} \mathbf{1}[ t_i \neq f_{\text{mp}} \newcirc h^* \rightcirc w^*(x_i) ] \Big) \tag{def. of $I$ and $\bar{I}$} \\
&\le \frac{1}{n} \Big( \sum_{i \in I} \mathbf{1}[ t_i \neq f_{\text{mp}} \newcirc h^* \rightcirc w^*(x_i) ] + |\bar{I}| \Big) \notag \\
&= \frac{1}{n} \Big( \sum_{i \in I} \mathbf{1}[ f_{\text{mp}}(s_i) \neq f_{\text{mp}} \newcirc h^* \rightcirc w^*(x_i) ] + |\bar{I}| \Big) \notag \\
&\le \frac{1}{n} \sum_{i \in I} \mathbf{1}[ s_i \neq h^* \rightcirc w^*(x_i) ] + \frac{|\bar{I}|}{n} \label{eq:proof_1}.
\end{align}
By definition of $\widehat{R}_{\mcS} (w^*, h^*)$, we also have:
\begin{align*}
&\widehat{R}_{\mcS} (w^*, h^*) = \frac{1}{n} \sum_{i=1}^n \mathbf{1}[ s_i \neq h^* \rightcirc w^*(x_i) ] \notag \\
&= \frac{1}{n} \Big( \sum_{i \in I} \mathbf{1}[ s_i \neq h^* \rightcirc w^*(x_i) ] + \sum_{i \in \bar{I}} \mathbf{1}[ s_i \neq h^* \rightcirc w^*(x_i) ] \Big) \\
&\ge \frac{1}{n} \sum_{i \in I} \mathbf{1}[ s_i \neq h^* \rightcirc w^*(x_i) ]. \notag
\end{align*}
Plug this into Eq.~\eqref{eq:proof_1} and note that $\text{MPA}(\mcT | \mcS) = |I|/n$, we have:
$\widehat{R}_{\mcT, \gamma} (w^*, k^*) \le \widehat{R}_{\mcS} (w^*, h^*) + |\bar{I}|/n \notag = \widehat{R}_{\mcS} (w^*, h^*) + 1 - \text{MPA}(\mcT | \mcS).$ 
\qedhere
\end{proof}

As a remark, the bound in Lemma~\ref{lemma:mpa} gets tighter when $\text{MPA}(\mcT | \mcS) \rightarrow 1$, that is, when $f_{\text{mp}}$ is more accurate. Using this lemma, we now prove the generalization bounds for the transferred model $k^* \rightcirc w^*$. Section~\ref{sec:fc_bound} below proves the bound for fully connected neural networks, while Section~\ref{sec:cnn_bound} proves the bound for convolutional neural networks.

\subsection{Generalization Bound for Fully Connected Networks}
\label{sec:fc_bound}

In this section, we consider target models $k \newcirc w$ that are deep neural networks parameterized by $w = \{ A^1, A^2, \ldots, A^L \}$ and $k = \{ A^{L+1}, A^{L+2}, \ldots, A^{L_T} \}$ such that:
\[ w(x) = \sigma_L(A^L\sigma_{L-1}(A^{L-1}\sigma_{L-2}(\ldots A^1(x)))), \text{ and } \] 
\[ k(w(x)) = A^{L_T}\sigma_{L_T-1}(A^{L_T-1}\sigma_{L_T-2}(\ldots A^{L+1}(w(x)))), \]
where $L$ is the depth of the neural network $w$, $L_T$ is the depth of the whole target network $k \newcirc w$, and $A^i \in \mathbb{R}^{W_i \times W_{i-1}}$ is the weight matrix at layer $i$ with $W_0 = d$, $W_L = r$, and $W_{L_T} = m_T$. In the above formulas, $\sigma_i : \mathbb{R}^{W_i} \rightarrow \mathbb{R}^{W_i}$ is a non-linear activation function that is assumed to be 1-Lipschitz.

We do not make any assumption regarding the form or architecture of the source head classifier $h$, except for Assumption~\ref{assumption}. Thus, our generalization bound in this section holds for all types of source head classifiers, including neural networks, logistic regression, support vector machines, etc. In practice, however, $h$ is usually chosen as a logistic regression or neural network for ease of implementation and better accuracy.

Following the notations in~\cite{ledent2021norm}, in our result, we write $\| A \|_{\text{Fr}}$, $\| A \|_{\sigma}$, and $\| A \|_{p,q}$ to denote respectively the Frobenius norm, the spectral norm, and the $(p,q)$-norm of a matrix $A$. We also write $A_{i, \sbullet}$ to denote the $i$-th row of $A$. We let $\bar{W} = \max_{i=1}^{L_T} W_i$ be the maximum width of the target neural network.
We now state and prove our generalization bound of the true risk $R_T(w^*, k^*) = \mathbb{P}(t \neq \argmax_i k^* \rightcirc w^*(x)_i)$ for this fully connected network setting in the theorem below.

\begin{theorem}
Assume we are given some fixed reference matrices $M^1, M^2, \ldots, M^{L_T}$ representing the initialized weights of the target network. Under Assumption~\ref{assumption}, with probability at least $1 - \delta$, for all margin $\gamma \in (0, \bgamma]$, we have:
\begin{align*}
&R_T (w^*, k^*) \le \widehat{R}_{\mcS} (w^*, h^*) + (1 - \text{MPA}(\mcT | \mcS)) ~+ \\
&\qquad \widetilde{\mathcal{O}}\Big( \frac{\max_{i=1}^n \|x_i\|_{\text{Fr}} \, \mathcal{F}_{\mathcal{A}}}{\gamma \sqrt{n}}\log(\bar{W}) + \sqrt{\frac{\log(1/\delta)}{n}} \Big),
\end{align*}
where $\mathcal{A} = (A^{*1}, A^{*2}, \ldots, A^{*L_T})$ is the weight matrices of the target fully connected neural network ${k^* \rightcirc w^*}$ trained using the deep transfer learning procedure in Section~\ref{sec:dtl_setting}, and \\[4pt]
${\hskip 5mm} \mathcal{F}_{\mathcal{A}} := L_T \max_i \| A^{*L_T}_{i, \sbullet} \|_{\text{Fr}} \Big( \prod_{i=1}^{L_T-1} \| A^{*i} \|_{\sigma} \Big)$ \\[2pt]
$\displaystyle {\hskip 15mm} \Big( \sum_{i=1}^{L_T-1}\frac{\| A^{*i} - M^i \|_{2,1}^{2/3}}{\| A^{*i} \|_{\sigma}^{2/3}} + \frac{\|A^{*L_T}\|_{\text{Fr}}^{2/3}}{\max_i \| A^{*L_T}_{i, \sbullet} \|^{2/3}_{\text{Fr}}} \Big)^{3/2}.$
\label{thrm:fcnn}
\end{theorem}

{\vskip 2mm}

\begin{proof}
Using Theorem 1 of~\cite{ledent2021norm}, with probability at least $1 - \delta$, for all margin $\gamma > 0$, we have:
\begin{align*}
&R_T (w^*, k^*) \le \widehat{R}_{\mcT, \gamma} (w^*, k^*) + \notag \\
&\qquad \widetilde{\mathcal{O}} \Big( \frac{\max_{i=1}^n \|x_i\|_{\text{Fr}} \, \mathcal{F}_{\mathcal{A}}}{\gamma \sqrt{n}}\log(\bar{W}) + \sqrt{\frac{\log(1/\delta)}{n}} \Big).
\end{align*}
Combining this with Lemma~\ref{lemma:mpa}, we have, with probability at least $1 - \delta$, for all margin $\gamma \in (0, \bgamma]$:
\begin{align*}
&R_T (w^*, k^*) \le \widehat{R}_{\mcS} (w^*, h^*) + (1 - \text{MPA}(\mcT | \mcS)) + \\
&\quad \widetilde{\mathcal{O}} \Big( \frac{\max_{i=1}^n \|x_i\|_{\text{Fr}} \, \mathcal{F}_{\mathcal{A}}}{\gamma \sqrt{n}}\log(\bar{W}) + \sqrt{\frac{\log(1/\delta)}{n}} \Big). \qedhere
\end{align*}
\end{proof}

\subsection{Generalization Bound for Convolutional Neural Networks}
\label{sec:cnn_bound}

We now consider target models $k \newcirc w$ that are convolutional neural networks. For these models, the matrices $A^1, A^2, \ldots, A^{L_T}$ are the filter matrices of the convolutional layers. Following~\cite{ledent2021norm}, for each filter matrix $A^i$, we can construct a corresponding larger convolutional matrix $\tilde{A}^i$ by repeating the weights of $A^i$ as many times as the filter $A^i$ is applied. The activation functions considered here are assumed to be either ReLU or max pooling.

In our result, $\bar{W}$ is the maximum number of neurons in a single layer before pooling, counting all the channels. For each layer $i$, $w_i$ is the spacial width of the layer after pooling, and $B_i$ is the maximum $l_2$ norm of any convolutional patch of the layer's activations over all inputs. For any layer $i \le L_T-1$, we also write $\|\tilde{A}^i\|_{\sigma'}$ to denote the maximum spectral norm of any matrix obtained by deleting, for each pooling window, all but one of the corresponding rows of $\tilde{A}^i$. For $i = L_T$, $\|\tilde{A}^{L_T}\|_{\sigma'} = \rho_{L_T} \max_{j} \|A^{L_T}_{j,\sbullet}\|_{\text{Fr}}$, with $\rho_{L_T}$ being the Lipschitz constant of the activation and pooling at layer $L_T$. More details of the notations can be found in~\cite{ledent2021norm}.

Theorem~\ref{thrm:cnn} below shows our generalization bound for this convolutional neural network setting.
Similar to Section~\ref{sec:fc_bound}, we do not restrict the form of the source head classifier $h$, so our result will also hold for all types of source head classifiers.

\begin{theorem}
Assume we are given some fixed reference matrices $M^1, M^2, \ldots, M^{L_T}$ representing the initialized weights of the target network's filter matrices. Under Assumption~\ref{assumption}, with probability at least $1 - \delta$, for all margin $\gamma \in (0, \bgamma]$, we have:
\begin{align*}
R_T (w^*, k^*) &\le \widehat{R}_{\mcS} (w^*, h^*) + (1 - \text{MPA}(\mathcal{T} | \mathcal{S})) ~+ \\
&\qquad \widetilde{\mathcal{O}} \Big( \frac{\mathcal{G}_{\mathcal{A}}}{\sqrt{n}}\log(\bar{W}) + \sqrt{\frac{\log(1/\delta)}{n}} \Big),
\end{align*}
where $\mathcal{A} = (A^{*1}, A^{*2}, \ldots, A^{*L_T})$ is the filter matrices of the target convolutional neural network $k^* \rightcirc w^*$ trained using the deep transfer learning procedure in Section~\ref{sec:dtl_setting}, and ${ \mathcal{G}_{\mathcal{A}}^{2/3} := \sum_{i=1}^{L_T} T_i^{2/3} }$, where for all $i \leq L_T-1$,
\[ T_i := B_{i-1} \|(A^{*i}-M^i)^\top\|_{2,1} \sqrt{w_i} \max_{U \le L_T}\frac{\prod_{u=i+1}^U \|\tilde{A}^{*u}\|_{\sigma'}}{B_U}, \]
and $\displaystyle T_{L_T} := B_{L_T-1} \|A^{*L_T}-M^{L_T}\|_{\text{Fr}}/\gamma$.
\label{thrm:cnn}
\end{theorem}

\begin{proof}
This proof is similar to the proof of our Theorem~\ref{thrm:fcnn} above, but replacing Theorem 1 of~\cite{ledent2021norm} by their Theorem 3, which states that with probability at least $1 - \delta$, for all $\gamma > 0$: 
\[
R_T (w^*, k^*) \le \widehat{R}_{\mcT, \gamma} (w^*, k^*) + \widetilde{\mathcal{O}} \big( \frac{\mathcal{G}_{\mathcal{A}}}{\sqrt{n}}\log(\bar{W}) + \sqrt{\frac{\log(1/\delta)}{n}} \big).
\]
The theorem holds by combining this with Lemma~\ref{lemma:mpa}.
\end{proof}

\section{Bounds for Different Training Inputs Setting}
\label{sec:diff-set}

Up until now, we have only considered the case where source and target datasets share the same input set. In this section, we extend our results to the case where the source and target datasets contain different inputs. Formally, let $\mcS = \{ (x_1, s_1), (x_2, s_2), \ldots, (x_n, s_n) \}$ be the source dataset where $(x_i, s_i) \in \mathbb{R}^d \times [m_S]$ is drawn iid from a joint distribution $\mathbb{P}_{X, S}$, and $\mcT = \{ (z_1, t_1), (z_2, t_2), \ldots, (z_p, t_p) \}$ be the target dataset where $(z_i, t_i) \in \mathbb{R}^d \times [m_T]$ is drawn iid from $\mathbb{P}_{Z, T}$. We also follow the deep transfer learning procedure in Section~\ref{sec:dtl_setting} and first train the optimal model $h^* \rightcirc w^*$ on the source data $\mcS$ using Eq.~\eqref{eq:source}. Then we freeze $w^*$ and train the target head classifier $k^*$ using Eq.~\eqref{eq:target} with the target dataset $\mcT$, where we will apply $w^*$ to the target inputs $\{ z_1, z_2, \ldots, z_p \}$ to get the representations $\{ w^*(z_1), w^*(z_2), \ldots, w^*(z_p) \}$.

To prove the generalization bounds, we will consider a new source dataset $\tilde{\mcS} := \{(z_i, h^* \rightcirc w^*(z_i))\}_{i=1}^p$ induced by $h^* \rightcirc w^*$ and the target inputs $\{ z_1, z_2, \ldots, z_p \}$. In essence, $\tilde{\mcS}$ contains the target inputs with ``dummy'' source labels generated by $h^* \rightcirc w^*$. This technique of using these dummy labels was previously employed to develop the LEEP transferability measure~\cite{nguyen2020leep}, and is useful for proving our bounds as well. With the new source dataset $\tilde{\mcS}$, we consider the majority predictor $f_{\text{mp}}$ constructed from $(\tilde{\mcS}, \mcT)$, as well as the corresponding majority predictor accuracy $\text{MPA}(\mcT | \tilde{\mcS})$. We still keep Assumption~\ref{assumption} in Section~\ref{sec:same_inp}, but adapt it to the new $h^* \rightcirc w^*$ and $f_{\text{mp}}$. The following lemma is the analogue of Lemma~\ref{lemma:mpa} for the different inputs setting.

\begin{lemma}
With the adapted Assumption~\ref{assumption}, for any $\gamma \in [0, \bgamma]$, we have:
$\widehat{R}_{\mcT, \gamma} (w^*, k^*) \le 1 - \text{MPA}(\mcT | \tilde{\mcS})$.
\label{lem:diffinput}
\end{lemma}

\begin{proof}
From~$\eqref{eq:mpa}$,~$\eqref{eq:fo}$, and the definition of $\tilde{\mcS}$, we have: 
\[
\text{MPA}(\mcT | \tilde{\mcS}) = \frac{1}{p} \sum_{i=1}^p \mathbf{1}[t_i = f_{\text{mp}} \newcirc h^* \rightcirc w^*(z_i)].
\]
By Assumption~\ref{assumption}, for any $\gamma \in [0, \bgamma]$, we have
\begin{align*}
&\widehat{R}_{\mcT, \gamma} (w^*, k^*) \le \widehat{R}_{\mcT} (w^*, f_{\text{mp}} \newcirc h^*) \\
&= \frac{1}{p} \sum_{i=1}^p \mathbf{1}[ t_i \neq f_{\text{mp}} \newcirc h^* \rightcirc w^*(z_i) ]
=  1 - \text{MPA}(\mathcal{T} | \tilde{\mathcal{S}} ). \qedhere
\end{align*}
\end{proof}
Similar to Section~\ref{sec:same_inp}, we derive the following generalization bounds, which are analogues of Theorems~\ref{thrm:fcnn} and~\ref{thrm:cnn}. The proofs of these theorems are similar to those of Theorems~\ref{thrm:fcnn} and~\ref{thrm:cnn}, with Lemma~\ref{lemma:mpa} being replaced by Lemma~\ref{lem:diffinput}.

\begin{theorem}
Assume we are given some fixed reference matrices $M^1, M^2, \ldots, M^{L_T}$ representing the initialized weights of the target network. Under the adapted Assumption~\ref{assumption}, with probability at least $1 - \delta$, for all margin $\gamma \in (0, \bgamma]$, with $\mathcal{F}_{\mathcal{A}}$ defined as in Theorem~\ref{thrm:fcnn}, we have:
$R_T (w^*, k^*) \le { 1 - \text{MPA}(\mcT | \tilde{\mcS}) + \widetilde{\mathcal{O}} \big( \frac{\max_{i=1}^p \|x_i\|_{\text{Fr}} \, \mathcal{F}_{\mathcal{A}}}{\gamma \sqrt{p}}\log(\bar{W}) + \sqrt{\frac{\log(1/\delta)}{p}} \big) }$.
\label{thrm:fcnn-diff}
\end{theorem}

\begin{theorem}
Assume we are given some fixed reference matrices $M^1, M^2, \ldots, M^{L_T}$ representing the initialized weights of the target network's filter matrices. Under the adapted Assumption~\ref{assumption}, with probability at least $1 - \delta$, for all margin $\gamma \in (0, \bgamma]$, with $\mathcal{G}_{\mathcal{A}}$ defined as in Theorem \ref{thrm:cnn}, we have: \\[5pt]
\noindent $R_T (w^*, k^*) \le 1 - \text{MPA}(\mcT | \tilde{\mcS}) + \widetilde{\mathcal{O}} \big( \frac{\mathcal{G}_{\mathcal{A}}}{\sqrt{p}}\log(\bar{W}) + \sqrt{\frac{\log(1/\delta)}{p}} \big).$
\label{thrm:cnn-diff}
\end{theorem}

\section{Discussions}

The technique used to prove our theorems is general and can be combined with other generalization bounds for deep neural networks. Although we proved our results using the norm-based bounds of~\cite{ledent2021norm}, we emphasize that our proof technique can also be used with other generalization bounds for neural networks, such as those of~\cite{bartlett2017spectrally}.

The bounds in Theorems~\ref{thrm:fcnn} and~\ref{thrm:cnn} depend on both the optimal source empirical risk $\widehat{R}_{\mcS} (w^*, h^*)$ and $\text{MPA}(\mcT | \mcS)$. These bounds get better when $\widehat{R}_{\mcS} (w^*, h^*) \rightarrow 0$ and $\text{MPA}(\mcT | \mcS) \rightarrow 1$. The bounds in Theorems~\ref{thrm:fcnn-diff} and~\ref{thrm:cnn-diff} do not contain the source empirical risk, since it has been indirectly measured in $\text{MPA}(\mcT | \tilde{\mcS})$ when we use $h^* \rightcirc w^*$ to construct $\tilde{\mcS}$.

From our results, we can see that $\text{MPA}(\mcT | \mcS)$ (or $\text{MPA}(\mcT | \tilde{\mcS})$ for the setting with different inputs) can be used as a transferability measure. Specifically, for well-trained and deep enough neural networks, it has been observed empirically~\cite{zhang2017understanding} that ${ \widehat{R}_{\mcS} (w^*, h^*) \approx 0 }$. Furthermore, the $\widetilde{\mathcal{O}}(\cdot)$ terms in our theorems are near $0$ for large enough $n$. In this case, our results imply that $\text{MPA}(\mcT | \mcS) \lessapprox 1 - R_T (w^*, k^*)$ or ${ \text{MPA}(\mcT | \tilde{\mcS}) \lessapprox 1 - R_T (w^*, k^*) }$. This means that $\text{MPA}(\mcT | \mcS)$ or $\text{MPA}(\mcT | \tilde{\mcS})$ lower bounds the expected accuracy of the transferred model $k^* \rightcirc w^*$, and thus can be used as a transferability measure. We now validate this observation empirically.

\section{Experiments}
\label{sec:experiment}

We show the usefulness of our theoretical bounds in practice by empirically illustrating the ability of MPA as a transferability measure on the large-scale Caltech-UCSD Birds-200 dataset~\cite{WelinderEtal2010}, which contains 11,788 images of 200 bird species labeled with 312 binary attributes. We keep the train-test split as provided in dataset, with 5,994 train images and 5,794 test images. We pick 4 attributes \emph{Curved Bill}, \emph{Iridescent Wings}, \emph{Brown Upper Parts} and \emph{Olive Under Parts} for training source models, and randomly choose 100 different attributes as target tasks. Regarding the model architecture, we use ResNet18~\cite{he2016deep} without the last fully connected layer as the feature extractor $w$. In all tests, we train our source model $h^* \rightcirc w^*$ and the transferred model $k^* \rightcirc w^*$ using the cross-entropy loss with batch size 32 and run the stochastic gradient descent optimizer with momentum for 40 epochs. The initial learning rate is set at 0.01 and is divided by 10 every 10 epochs.

Following the settings in~\cite{tran2019transferability, nguyen2020leep}, we estimate the correlations between the MPA scores and the actual test accuracies of the transferred models to evaluate the relationship between these two quantities. High correlations mean the MPA score is a good measure for comparing test accuracies of the transferred models, and thus is a good transferability measure. For the 4 source tasks above with 100 randomly chosen target tasks, our experimental results give the following Pearson correlation coefficients: 0.9534 (Curved Bill), 0.9452 (Iridescent Wings), 0.9484 (Brown Upper Parts), and 0.9611 (Olive Under Parts). These coefficients show that the MPA scores and the test accuracies are highly positive correlated with statistical significance ($p < 10^{-4}$), which clearly indicates that the MPA is a reliable transferability measure for estimating the performance of transferred models.

\section{Conclusion}

We proved novel generalization bounds for transfer learning of deep neural networks using a new quantity, the majority predictor accuracy, that can be computed easily and efficiently from data. We showed the usefulness of our bounds in practice by demonstrating that the majority predictor accuracy can be used for estimating the effectiveness of deep transfer learning. Our theory can potentially be extended to analyze more complex transfer scenarios such as continual learning~\cite{nguyen2019toward}.

\bibliographystyle{plain}
\bibliography{isita22}

\end{document}